\documentclass[twoside,11pt]{article}

\usepackage{jmlr2e}

\usepackage[utf8]{inputenc} 
\usepackage[T1]{fontenc}    
\usepackage{hyperref}       
\usepackage{url}            
\usepackage{booktabs}       
\usepackage{amsfonts}       
\usepackage{nicefrac}       
\usepackage{microtype}      

\usepackage{amsmath}   
\usepackage{lastpage}
\usepackage{text comp} 
\usepackage{mathtools} 
\usepackage[inline]{enumitem}
\usepackage{todonotes}
\usepackage{bm}
\usepackage{mathstyle} 
\usepackage[cal = esstix, scr = boondoxo, scrscaled = 1]{mathalfa}
\usepackage[acronym, shortcuts]{glossaries}
\usepackage{multirow}
\usepackage{tabu}
\usepackage{extarrows}
\usepackage[all]{xy} 
\usepackage{bbm} 
\usepackage{needspace} 
\usepackage{bookmark} 
\usepackage{xspace} 
\usepackage{needspace}
\newtheorem{false_thm}[theorem]{Claim}

\graphicspath{{figures/}}

\def\iftodo{\iffalse} 
\def\ifcomments{\iffalse} 

\newcommand{\myQED}{\rule{1.5ex}{1.5ex}}  
\newenvironment{proof_no_qed}{\par\noindent{\bf Proof\ }}{}

\newcommand{\Todo}[2][]{
	\iftodo
		\ifthenelse{\equal{#1}{}}{\todo{#2}}{\todo[#1]{#2}}
	\fi
}

\renewcommand{\k}{k}

\newcommand{\then}{\ensuremath{\Rightarrow}}

\newcommand{\xx}{\bm{x}}
\newcommand{\yy}{\bm{y}}
\newcommand{\xxi}{\bm{\xi}}
\newcommand{\zzeta}{\bm{\zeta}}

\newcommand{\sX}{\mathcal{X}}
\newcommand{\R}{\mathbb{R}}

\newcommand{\sK}{\mathcal{K}}  
\newcommand{\sA}{\mathcal{A}}
\newcommand{\sO}{\mathcal{O}}
\newcommand{\sM}{\mathcal{M}}
\newcommand{\sD}{\mathcal{D}}
\newcommand{\sB}{\mathcal{B}}

\newcommand{\HH}{\mathcal{H}}
\newcommand{\Hk}{\HH_{\k}}

\newcommand{\F}{\mathcal{F}}

\renewcommand{\L}{\mathrm{L}}
\newcommand{\C}[2]{\mathscr{C}^{#1}_{#2}}

\newcommand{\Mf}{\sM}

\newcommand{\Mprob}{\mathcal{P}}
\newcommand{\Mplus}{\sM_{\scriptscriptstyle{+}}}

\newcommand{\diff}{\mathop{} \! \mathrm{d}} 

\newcommand{\ipdk}[2]{\ensuremath{\left \langle #1 \, , \, #2 \right \rangle_{\k}}}
\newcommand{\norm}[1]{\ensuremath{\left \| #1 \right \|}}
\newcommand{\normk}[1]{\ensuremath{\left \| #1 \right \|_{\k}}}
\newcommand{\back}{\backslash}

\providecommand{\function}[5]{} 
\renewcommand{\function}[5]{
	\ensuremath{
	\mathchoice{
	\ifthenelse{\equal{#1}{}}
            	{
            	\begin{array}[t]{ccl}
            		\ifthenelse{{\equal{#2}{}}}{
            		#4 & \longmapsto & #5}{
            		#2 & \longrightarrow & #3
            		 \ifthenelse{\equal{#4}{}} {} {\\
            		#4 & \longmapsto & #5}}
            	\end{array} \!
            	}
            	{
            	\begin{array}[t]{lccl}
            	#1 : 
            		\ifthenelse{{\equal{#2}{}}}{
            		& #4 & \longmapsto & #5}{
            		& #2 & \longrightarrow & #3
            		 \ifthenelse{\equal{#4}{}} {} {\\
            		 & #4 & \longmapsto & #5}}
            		 \end{array} \!
            	}
	}
	{
		\ifthenelse{\equal{#1}{}}
		{
			\ifthenelse{{\equal{#2}{}}}{
			{#4 \mapsto \, #5}}{
			{#2 \rightarrow \, #3}
			 \ifthenelse{\equal{#4}{}} {} {
			, \; {#4 \mapsto \, #5}}}
		}
		{
		#1 :%
			\ifthenelse{{\equal{#2}{}}}{
			{#4 \mapsto  \, #5}}{
			{#2 \rightarrow \, #3}
			 \ifthenelse{\equal{#4}{}} {} {
			, \; {#4 \mapsto \, #5}}}
		}	
	}
    {}{}}%
}

\setkeys{glslink}{hyper=false}

\renewcommand*{\CustomAcronymFields}{%
  name={\the\glsshorttok},
  description={\the\glslongtok},
  first={\noexpand{\the\glslongtok}\space(\the\glsshorttok)},%
  firstplural={\noexpand{\the\glslongtok\noexpand\acrpluralsuffix}\space(\the\glsshorttok)},%
  text={\the\glsshorttok},%
  plural={\the\glsshorttok\noexpand\acrpluralsuffix}%
}

\SetCustomStyle

\newacronym{ispd}{$\int$s.p.d.\@}{integrally strictly positive definite}
\newacronym{spd}{s.p.d.\@}{strictly positive definite}
\newacronym{cpd}{c.p.d.\@}{conditionally positive definite}
\newacronym{iff}{iff}{if and only if}
\newacronym{wrt}{w.r.t.\@}{with respect to}
\newacronym{lcv}{loc.\ cv.\@}{locally convex}

\newcommand{\ispd}{\gls{ispd}\xspace}

\newcommand{\cj}[1]{\ifcomments \textcolor{red}{CJ: #1} \fi}

\allowdisplaybreaks
\setenumerate{label=(\roman*), noitemsep}
\setitemize{label=$\triangleright$, noitemsep}
\setdescription{labelindent=1em,leftmargin=4em, style=nextline}

\jmlrheading{1}{2021}{1-48}{4/00}{10/00}{simon21}{Simon-Gabriel et al.}

\ShortHeadings{Metrizing Weak Convergence with MMDs}{Simon-Gabriel et al.}
\firstpageno{1}

\begin{document}

\title{Metrizing Weak Convergence\\ with Maximum Mean Discrepancies}


\author{\name Carl-Johann Simon-Gabriel \email cjsg@ethz.ch \\
  \addr Institute for Machine Learning\\
  ETH Zürich, Switzerland
  \AND
  \name Alessandro Barp \email ab2286@cam.ac.uk \\
  \addr Department of Engineering\\
  University of Cambridge, Alan Turing Institute, United Kingdom
  \AND
  \name Bernhard Sch{\"o}lkopf \email bs@tue.mpg.de\\
  \addr Empirical Inference Department\\
  MPI for Intelligent Systems, T{\"u}bingen, Germany
  \AND
  \name Lester Mackey \email lmackey@microsoft.com\\
  \addr Microsoft Research\\
  Cambridge, MA, USA}


\maketitle

\begin{abstract}
   This paper characterizes the maximum mean discrepancies (MMD) that metrize
   the weak convergence of probability measures for a wide class of kernels.
   More precisely,
   we prove that, on a
   locally compact, non-compact, Hausdorff space, the MMD of a bounded continuous Borel measurable kernel $\k$, whose RKHS-functions vanish at infinity (i.e.,\ $\Hk
   \subset  \C{}{0}$), metrizes the weak convergence of probability measures if
   and only if $\k$ is continuous and \ispd over all signed, finite, regular
   Borel measures. 
   We also correct a prior result of 
   \citet[JMLR][Thm.~12]{simon18kde}
   by showing that
   there exist both bounded continuous \ispd kernels that do not metrize weak convergence and bounded continuous non-\ispd kernels that do metrize it.
\end{abstract}


\glsresetall  

\section{Introduction}


Although the mathematical and statistical literature has studied kernel mean
embeddings (KMEs) and maximum mean discrepancies (MMDs) at least since the
seventies \citep{guilbart78etude}, the machine learning community re-discovered
and applied them only since the late 2000s \cite{smola07hilbert}. A KME with
reproducing kernel $\k$ is a map from measures $\mu$ --~in particular
probability distributions~--  to  functions $f_{\mu}$ in the reproducing
kernel Hilbert space (RKHS) $\Hk$ of $\k$. The RKHS distance between two
embeddings then yields a semi-metric $d_{\k}$ on measures, called the maximum
mean discrepancy (MMD), which can be used to compare two measures or
distributions $\mu$ and $\nu$: $d_{\k}(\mu, \nu) := \normk{f_\mu - f_\nu}$.

Their theoretical tractability and computational flexibility has allowed MMDs
to flourish in many areas of machine learning that require comparing
probability distributions, such as two-sample testing (compare two discrete
distributions \cite{gretton12kernel}), sample quality measurement and goodness-of-fit testing (compare a
discrete distribution to a reference distribution, as in
\citealt{chwialkowski16kernel,liu16kernelized,gorham17measuring,jitkrittum17linear,huggins2018random}),
generative model fitting (compare distributions of fake and real data; see
\citealt{dziugaite15training,sutherland17generative,feng2017learning,pu2017vae,briol2019statistical}), 
de novo sampling and quadrature \citep{chen10super,huszar2012optimally,liu16stein,chen18stein,futami2019bayesian,chen2019stein}, 
importance sampling \citep{liu17black-box,hodgkinson2020reproducing}, 
and thinning \citep{riabiz2020optimal}.

For most applications, one seeks a kernel $\k$ whose MMD can separate all
probability distributions $P, Q$, meaning that, $d_{\k}(P,Q) = 0$ (if and) only
if $Q=P$. Such kernels are said to be \emph{characteristic} (to the set of
probability distributions $\Mprob$). If for example we optimize a parametric
distribution $Q$ to match a target $P$ by minimizing their MMD $d_{\k}(P,Q)$,
it is rather natural to require that it be minimized only if $Q$ perfectly
matches $P$, i.e.\ $Q=P$. Another natural, but a priori stronger requirement,
is that when $Q$ gets closer to $P$ in MMD, for example, if $d_{\k}(Q, P) \to
0$, we would like $Q$ to ``truly'' converge to $P$, where ``truly'' means
``for some other standard and/or more familiar notion of convergence''.

Although several standard notions may come to mind --~convergence in
KL-divergence, in total variation or in Hellinger distance~--, many are too
strong for our purposes which often require handling discrete data. For
example, even if $\xx \to \xxi$, the Dirac masses $\delta_{\xx}$ will not
converge to $\delta_{\xxi}$ in total variation or KL-divergence unless $\xx$ is
eventually equal to $\xxi$. Said differently, a sequence of deterministic
variables would not converge in total variation unless it was eventually
constant. Since in practice MMDs are frequently used to compare samples or
empirical (hence discrete) distributions, it comes as no surprise that MMD
convergence cannot, in general, ensure these strong types of convergence.
Instead we will opt for a standard, yet comparatively weak notion of
convergence, known as \emph{weak}  or \emph{narrow convergence} or
\emph{convergence in distribution}. Specifically, the central question of this
paper will be
\begin{quote}
    When is convergence in MMD metric equivalent to weak convergence on
    $\Mprob$?
\end{quote}
In that case, we will say that the kernel $\k$ \emph{metrizes the weak
convergence of probability measures}.
This question lies at the heart of the learning applications described above,
as the quality of these inferences depends on the metrization properties of the
chosen kernel \citep{zhu2018universal,zhu2019asymptotically,ansari2019characteristic,li17mmdgan}.
When the kernel MMD fails to reflect the convergence of distributions, the results are at best inaccurate and at worst invalid.

\subsection{Previous results, contributions and paper structure}

The aforementioned question was studied as early as \citeyear{guilbart78etude} by \citet{guilbart78etude} in his
thesis.  On separable metric spaces, he
characterized the kernels for which weak convergence implies convergence in MMD
(Thm.1.D.I). Conversely, he showed that, in some cases, MMD convergence can also
imply weak convergence, meaning that there do exist kernels that metrize weak
convergence. He provided a concrete recipe to construct such kernels (Thm.1.E.I
\& Lem.3.E.I) and used it to exhibit some examples. However,
\citet{guilbart78etude} did not characterize these kernels, and left most
standard kernels (Gaussian, Laplacian, etc.) aside.

These initial results went largely unnoticed by the ML community, and it is only much
later, with the emergence and the new applications of MMDs in applied
statistics, that the important question of weak convergence metrization
re-surfaced.  \Citet{sriperumbudur10hilbert} in particular presented sufficient
conditions under which the MMD metrizes weak convergence when the underlying
input space is either $\mathbb{R}^d$ (Thm.24) or a compact metric space
(Thm.23). \Citet[][Thm.3.2]{sriperumbudur13optimal} then considerably improved
these results and showed the following theorem.

\begin{theorem}[\citealt{sriperumbudur13optimal}]\label{thm:sri}
A continuous, bounded, \ispd kernel over a locally compact Polish
space $\sX$ such that $\Hk \subset \C{}{0}$ metrizes weak convergence.
\end{theorem}

Let us explain and discuss this result, for it will help understanding our own
results. First, the theorem assumes that the underlying input space is locally
compact \emph{and} Polish. Both assumptions taken separately are extremely
general: all topological manifolds (f.ex.\ $\R^d$) and all discrete spaces are
locally compact; while all separable, complete, metric spaces are, by definition,
Polish, which includes any separable Banach space. This generality made
locally compact spaces on the one side and Polish spaces on the other standard
alternatives to do general measure and probability theory on. However, when
both assumptions get combined, they typically become quite restrictive. A
Banach space, for example, is locally compact only if it has finite dimension.
Therefore, combining both assumptions yields an important constraint that limits the applicability of the result: one would hope for  one or
the other, but not both.

Second, $\Hk \subset \C{}{0}$ means that the RKHS functions $f$ are assumed to
be continuous and vanish at infinity, i.e., for any $\epsilon > 0$, there
exists a compact $\sK \subset \sX$ for which $\sup_{\sX \back \sK} |f| \leq
\epsilon$. Many standard kernels satisfy this assumption which is typically
easy to verify (see Lem.\ref{lem:Hk_C0} below). This assumption is also rather
natural for problems involving finite measures on locally compact spaces $\sX$,
because the (continuous) dual of $\C{}{0}$ can be identified with the set of
finite signed measures on $\sX$ (Riesz representation theorem). However, it is
often inadequate on Polish spaces, because $\C{}{0}$ can typically be very
small. For example, on an infinite dimensional Banach space, $\C{}{0}$ contains
only the null function.  This suggests that, in a first step, it might be more
natural to get rid of the Polish assumption than the locally compact
assumption.

Third, the theorem assumes that the kernel is \ispd, meaning that its MMD
separates all finite signed measures $\Mf$: for any $\mu, \nu \in \Mf$,
$d_{\k}(\mu, \nu) = 0$ only if $\mu = \nu$. It is easy to see that an MMD that
metrizes weak convergence on the set of probability measures $\Mprob$, must
separate $\Mprob$. But by assuming that it even separates $\Mf$, which is
bigger than $\Mprob$, \citet{sriperumbudur13optimal}'s Thm.2 leaves open the
case of any MMD that separates $\Mprob$ but not $\Mf$.

In 2018, \citet[][Thm.12]{simon18kde} seemed to finally address all weaknesses
mentioned above by characterizing the metrization of weak convergence of
probability measures on locally compact spaces as follows.
\begin{false_thm}[\citealt{simon18kde}]\label{thm:simon}
    On a locally compact Hausdorff space, a bounded, Borel measurable kernel
    metrizes the weak convergence of probability measures if and only if it is continuous
    and characteristic (to the set of probability measures).
\end{false_thm}
This statement weakens Theorem~\ref{thm:sri}'s sufficient condition from
separation of $\Mf$ (\ispd kernel) to separation of $\Mprob$ (characteristic
kernel), which, as discussed, immediately yields the converse direction. It
gets rid of the Polish assumption and, surprisingly, also drops the assumption
$\Hk \subset \C{}{0}$.

\paragraph{Contributions.} 
Unfortunately, Claim~\ref{thm:simon} turns out to be wrong when the input
space $\sX$ is not compact. 
Our main result, Theorem~\ref{thm:charac},
provides a correction under the additional assumption that $\Hk \subset \C{}{0}$.
Crucially, we find that the compact and non-compact case are inherently
different.
Metrizing weak convergence on non-compact spaces requires \emph{strictly}
stronger conditions, since the MMD needs to separate, not only the probability measures --~as in the
compact case or in Claim~\ref{thm:simon}~-- but all finite
signed measures.
Put differently,  Theorem~\ref{thm:charac} drops the Polish
assumption from Theorem~\ref{thm:sri} and proves that its converse -- which is
too strong when $\sX$ is compact (see Thm.~\ref{thm:compact} \&
Prop.~\ref{prop:exist}) -- does hold when $\sX$ is non-compact.
An important implication is that any $\C{}{0}$ kernel that maps a probability measure to $0$ \emph{fails} to metrize weak convergence; in particular, this establishes that large classes of Stein kernels are unable to metrize convergence (see Rem.~\ref{rem:significance_of_main}).

Additionally, Corollary~\ref{cor:counterexamples} shows that Theorem~\ref{thm:charac}
does not hold without the assumption $\Hk \not \subset \C{}{0}$, while
Corollary~\ref{cor:extension} provides a \emph{sufficient} condition to metrize
weak convergence when $\Hk \not \subset \C{}{0}$.
Our results also complete the findings of \citet{chevyrev18signature}, who
constructed a counter-example showing that Claim~\ref{thm:simon} does not
hold on Polish spaces.
Overall, our findings show that the old quest to characterize weak-convergence
metrizing MMDs --~which we close under the quite general assumption that $\sX$
is locally compact and $\Hk \subset \C{}{0}$~-- depends in much more subtle
ways on the properties of the underlying space $\sX$ (being compact or not,
Polish or not, etc.) and the kernel $\k$ ($\Hk$ contained in $\C{}{0}$ or not)
than was previously thought.

\paragraph{Paper structure.} Section~\ref{sec:preliminaries} fixes notations
and makes a few important reminders and remarks. Section~\ref{sec:sufficient}
then extends \citet{sriperumbudur13optimal}'s Theorem~\ref{thm:sri} and gives a
general sufficient condition to metrize weak convergence when $\Hk \subset
\C{}{0}$. We then investigate whether this condition is also necessary, first
when the input space $\sX$ is compact (Sec.~\ref{sec:compact}), where it turns
out to be too strong (Thm.~\ref{thm:compact}); then when $\sX$ is not compact,
but locally compact (Sec.~\ref{sec:locally_compact}), in which case the
sufficient condition turns out to be necessary (Thm.~\ref{thm:charac}). We
finish with a few results in the general case (Sec.~\ref{sec:general}), when
$\Hk \not \subset \C{}{0}$: first a negative result
(Cor.~\ref{cor:counterexamples}) showing that the assumption $\Hk \subset
\C{}{0}$ cannot be dropped without replacement; then a result that
generalizes the condition $\Hk \subset \C{}{0}$. Section~\ref{sec:conclusion}
concludes.

\subsection{Notation, definitions, reminders}\label{sec:preliminaries}

\paragraph{Notations.} We use the letter $\k$ to denote a (reproducing) kernel (i.e.\ a
positive definite function) over a \emph{locally compact Hausdorff space $\sX$}
and $\Hk$ denotes its RKHS. $\C{}{b}$ is the space of bounded, continuous and
real valued \footnote{Our results extend to complex valued functions modulo
some obvious slight modifications.} functions $f$ over $\sX$. $\C{}{0}$ is its
subspace of functions that vanish at infinity, i.e.\ such that for any
$\epsilon > 0$, there exists a compact $\sK \subset \sX$ such that $|f| \leq
\epsilon$ on $\sX \back \sK$. We denote its (continuous) dual $(\C{}{0})'$ by
$\Mf$, which, by the Riesz representation theorem, can be identified with the
set of signed, $\sigma$-additive, finite, regular Borel measures. We recall
that a signed, $\sigma$-additive measure $\mu$ is said to be \emph{regular} if,
for any Borel measurable set $\sA$ and any $\epsilon > 0$, there exists a
compact $\sK$ and an open set $\sO$ in $\sX$ such that $\sK \subset \sA \subset
\sO$, $| \mu(\sA)-\mu(\sK)| \leq \epsilon$ and $|\mu(\sO) - \mu(\sA)| \leq
\epsilon$.  $\L(\mu)$ denotes the set of $\mu$-integrable functions (i.e.\
verifying $\int_{\sX} |f| \diff |\mu| < \infty$) and for any such function $f$
we write $\mu(f) := \int_{\sX} f \diff \mu$.
We denote by $\Mplus$, $\Mprob$ and
$\sM^{0}$ the subsets of $\Mf$ consisting of non-negative measures, of probability measures,
and of signed measures $\mu$ such that $\mu(\sX) = 0$ respectively.

\paragraph{Definition of KMEs and MMDs.} For a continuous, bounded kernel $\k$
and any $\mu \in \Mf$, $\int_{\sX} \normk{\k(.,\xx)} \diff \mu = \int_{\sX}
\sqrt{\k(\xx,\xx)} \diff \mu(\xx) < \infty$. By standard properties of the
so-called \emph{Bochner integral} \citep{schwabik05topics}, the
(Bochner-)integral 
\[
    f_{\mu}(\cdot) := \int_{\sX} \k(., \xx) \diff \mu(\xx)
\]
is a
well-defined function in the RKHS $\Hk$ of $\k$, and all functions $f \in \Hk$
are $\mu$-integrable and verify what we call the \emph{Pettis property}:
$\mu(f) = \ipdk{f_{\mu}}{f}$. In particular, for any $\mu, \nu \in \Mf$,
\[
    \ipdk{\mu}{\nu} := \ipdk{f_\mu}{f_\nu} = \mu \otimes \nu (\k)
        \qquad \text{and} \qquad
        \normk{\mu}^2 = \mu \otimes \mu (\k) \ ,
\]
where $\mu \otimes \nu$ denotes the (tensor) product measure between $\mu$ and
$\nu$. The maximum mean discrepancy (MMD) $d_{\k}(\mu, \nu)$ between $\mu$ and
$\nu$ is then defined as the RKHS distance between their embeddings:
\[
    d_{\k}(\mu,\nu) := \normk{\mu - \nu} = \normk{f_\mu - f_\nu} \ .
\]

\paragraph{Why bounded kernels?} In all our results, we will assume that the
kernel $\k$ is bounded. One may wonder if those results could be generalized to
unbounded kernels. To do so, one would need a definition of KMEs and MMDs that
allows unbounded kernels. Such generalizations do exist (see f.ex.\ Def.~1 in
\citealt{simon18kde}), but they all at least require that $\Hk \subset \L(\mu)$
for any embeddable measure $\mu$.  But if $\k$ is unbounded, then $\Hk$
contains an unbounded function $f$ \citep[][Cor.~3]{simon18kde}, and therefore,
it is easy to construct a probability measure $P$ such that $f \not \in \L(P)$.
So $P$ does not embed into $\Hk$ and the MMD is not defined over all
probability measures and cannot, a fortiori, metrize weak convergence there. 

\paragraph{Equivalence of universal, characteristic and \ispd kernels.} Let $\F$ be a normed
set of functions and $\sD$ a subset of $\Mf$. A kernel $\k$ is said to be
\emph{universal to $\F$} if $\Hk$ is a dense subset of $\F$. It is
\emph{characteristic to $\sD$} --~or just \emph{characteristic} when $\sD =
\Mprob$~-- if the KME is well-defined and injective over $\sD$. It is said to be
\emph{\acrfull{ispd} to $\sD$} --~or just \emph{\ispd} when $\sD = \Mf$~-- if
its MMD separates all measures in $\sD$. It will be useful to remember that a
kernel is universal to $\F$ (f.ex.\ to $\C{}{0}$) if and only if it is characteristic to
its dual ($(\C{}{0})' = \Mf$) \citep[][Thm.6 \& Tab.1]{simon18kde}. Also, it is
characteristic to a set if and only if it is \ispd to that same set (which is
almost immediate to see). The distinction between characteristicness and \ispd
is mostly due to historical reasons. We advice to simply think in terms of
separation of $\sD$.

\section{Sufficient conditions to metrize weak convergence}
\label{sec:sufficient}

We start with a lemma that extends Theorem~\ref{thm:sri}. Its main message is
the same: bounded, continuous, \ispd kernels metrize weak convergence of
probability measures. But, importantly, it drops the Polish assumption and adds
a few interesting details. For one thing, it shows that weak and MMD
convergence also coincide with (the a priori even weaker) vague and weak RKHS
convergence. 
For another, it adds a form of converse: weak convergence implies MMD
convergence if \emph{and only if} the kernel is bounded and continuous. Since
most usual kernels are bounded and continuous, this lemma also confirms what we
mentioned earlier: convergence in MMD is often rather weak and can, at best,
metrize weak convergence, but not convergence in total variation or KL
divergence (since those are known to be strictly stronger than weak
convergence).

\begin{lemma}\label{lem:equi_topos}
    Let $\k$ be an \ispd kernel such that $\Hk \subset \C{}{0}$ and let
    $(P_\alpha)$ (sequence or net) and $P$ be probability measures. If $\k$ is
    continuous, then the following are equivalent.
    \begin{enumerate}[label={\upshape(\roman*)}]
        \item \label{it1} $\normk{P_\alpha - P} \to 0$ \qquad \qquad \quad
        \qquad \: \ (convergence in strong RKHS topology)

        \item \label{it2} $P_\alpha(f) \to P(f)$ for all $f \in \Hk$
        \qquad (convergence in weak RKHS topology)

        \item \label{it3} $P_\alpha(f) \to P(f)$ for all $f \in \C{}{0}$
        \qquad (convergence in weak-$*$ or vague topology)

        \item \label{it4}  $P_\alpha(f) \to P(f)$ for all $f \in \C{}{b}$
        \qquad (convergence in weak topology)
    \end{enumerate}
    Conversely, if \ref{it4} implies \ref{it1} for any probability measures
    $(P_\alpha)$ and $P$, then $\k$ is continuous.
\end{lemma}

When \ref{it1} and \ref{it4} are equivalent for all sequences of probability
measures, we say that \emph{$\k$ metrizes the weak convergence of probability
measures}.

\begin{proof}
    Since $\Hk \subset \C{}{0} \subset \C{}{b}$, \ref{it4} \then \ref{it3}
    \then \ref{it2}. Moreover, strong RKHS convergence implies weak RKHS
    convergence, that is \ref{it1}
    \then \ref{it2}, since $P(f) = \ipdk{P}{f}$ for any $f\in\Hk$.
    Now assume $\k$ is continuous. If \ref{it4}, then the
    product measures $P_\alpha \otimes P$, $P \otimes P_\alpha$ and $P_\alpha
    \otimes P_\alpha$ converge weakly to $P \otimes P$
    \citep[Thm.~2.3.3]{berg84harmonic}. Hence 
    \[
        \normk{P_\alpha - P}^2 
            = P_\alpha \otimes P_\alpha(\k) + P \otimes P (\k)
                - P_\alpha \otimes P (\k) - P \otimes P_\alpha (\k) \to 0 \ ,
    \]
    i.e.\ \ref{it4} \then \ref{it1}. Summing up
    so far: \ref{it4} \then \ref{it1} \then \ref{it2} and \ref{it4} \then
    \ref{it3} \then \ref{it2}.

    Conversely, assume \ref{it2}. Since $\k$ is \ispd and $\Hk \subset
    \C{}{0}$, by Cor.~3 and Thm.8 in \citet{simon18kde}, $\Hk$ is dense in
    $\C{}{0}$. And since $\Mprob$ is a bounded subset of the dual $\Mf$ of
    $\C{}{0}$ (which is a Banach, hence barreled space), by Thm.33.2 in
    \citet{treves67tvs}, $\Mprob$ is equicontinuous in vague topology. So, by
    Prop.32.5 in \citet{treves67tvs}, \ref{it2} implies vague convergence, i.e.\
    \ref{it3}.
    Cor.~2.4.3 in \citet{berg84harmonic} then yields \ref{it4}. Hence the
    equivalence of \ref{it1} to \ref{it4}.

    Now assume \ref{it4} \then \ref{it1} on $\Mprob$, and suppose that $\xx \to
    \xxi$ and $\yy \to \zzeta$ in $\sX$. Then the Dirac point masses
    $\delta_\xx$ and $\delta_\yy$ converge weakly to $\delta_\xxi$ and
    $\delta_\zzeta$, which, by assumption, implies convergence in RKHS norm.
    Since the inner product is continuous (for the RKHS norm/topology), we get
    \[
        \k(\xx, \yy) = \ipdk{\delta_\xx}{\delta_\yy} 
            \to \ipdk{\delta_\xxi}{\delta_\zzeta} = \k(\xxi, \zzeta) \ ,
    \]
    so $\k$ is continuous.
\end{proof}

\begin{remark}
    The proof shows that \ref{it2} and \ref{it3} are even equivalent on any
    bounded subset of $\Mf$ \citep[][Prop.32.5]{treves67tvs} (even without
    continuity of $\k$) and that \ref{it1}--\ref{it4} are actually equivalent
    on any bounded subset of $\Mplus$ whenever $P_\alpha(\sX) \to P(\sX)$
    (which is always true for probability measures).
\end{remark}

The previous lemma gives  sufficient conditions to metrize weak convergence. We
now investigate whether they are necessary. To do so, we have to distinguish
the case where the input space $\sX$ is compact and where the conditions turn
out to be too strong, from the one where $\sX$ is locally compact but not
compact (and $\Hk \subset \C{}{0}$), where they are necessary.

\section{Necessary condition for compact input space \texorpdfstring{$\sX$}{X}}
\label{sec:compact}

When the underlying space $\sX$ is not just locally compact but compact, the
equivalence given in Claim~\ref{thm:simon} actually turns out to
hold: contrary to the general case, here, a continuous kernel only needs to
separate the probability measures to also metrize their weak convergence. The
reason for this difference is essentially that, because $\sX$ is compact,
measures cannot diffuse to 0 at infinity (see
Section~\ref{sec:locally_compact}).

\begin{theorem}\label{thm:compact}
    On a compact Hausdorff space, a bounded, measurable kernel metrizes the
    weak convergence of probability measures if and only if it is continuous
    and characteristic to $\Mprob$.
\end{theorem}

\begin{proof}
    If $\k$ metrizes weak convergence, then the RKHS metric needs to separate
    all probability measures, i.e.\ $\k$ is characteristic to $\Mprob$. And the
    last sentence of Lem.~\ref{lem:equi_topos} shows that $\k$ is continuous.
    Conversely, if $\k$ is characteristic to $\Mprob$, then the kernel $\kappa
    := \k + 1$ is \ispd \citep[][Thm.~8]{simon18kde}. Also, since $\k$ is
    continuous, $\kappa$ is continuous. Thus $\HH_{\kappa}$ is a continuous
    subspace of $\C{}{} = \C{}{b} = \C{}{0}$ (\citealt[][Cor.~4]{simon18kde} and
    compactness). By Lem.~\ref{lem:equi_topos}, $\kappa$ metrizes weak
    convergence on $\Mprob$, and by Thm.~8 of \citet{simon18kde}, $\kappa$ and
    $\k$ induce the same metric on $\Mprob$.
\end{proof}

What is surprising here is that, on a compact space and for a continuous
kernel, it suffices to separate probability measures to also metrize their weak
convergence, which, a priori, may have seemed a strictly stronger requirement.
We will see that when $\sX$ is not compact, this need not be the case.

\section{Necessary condition when \texorpdfstring{$\sX$}{X} is locally compact,
non-compact and \texorpdfstring{$\Hk \subset \C{}{0}$}{Hk in C0}}
\label{sec:locally_compact}

Since the condition $\Hk \subset\C{}{0}$ is at the heart of this section, we
would like to remind the reader that, by the following lemma
\citep[][Cor.~3]{simon18kde}, it is satisfied by many standard kernels:
Gaussian, Laplacian, Matern, inverse multi-quadratic kernels, etc.
\begin{lemma}\label{lem:Hk_C0}
    $\Hk \subset \C{}{0}$ if and only if $\k$ is bounded (i.e.\ $\sup_{\xx \in \sX}
    \k(\xx, \xx) < \infty$) and for all $\xx \in \sX$, $\k(\xx,.) \in \C{}{0}$.
\end{lemma}

We now turn to our main theorem, which corrects Claim~\ref{thm:simon} when
$\sX$ is non-compact and $\Hk \subset \C{}{0}$.

\begin{theorem}\label{thm:charac}
    Suppose that $\sX$ is not compact and that $\Hk \subset \C{}{0}$. Then $\k$
    metrizes weak convergence of probability measures if and only if $\k$ is
    continuous and \ispd (i.e.\ characteristic to $\Mf$).
\end{theorem}

We see that, contrary to the compact case, it is not enough to separate all
probability measures $\Mprob$ to metrize their weak convergence: $d_{\k}$ must
separate all finite measures $\Mf$, which strictly contains $\Mprob$.
Moreover, Proposition~\ref{prop:exist} below confirms that there are indeed
kernels that separate $\Mprob$ but not $\Mf$. Hence, Theorems~\ref{thm:compact}
and~\ref{thm:charac} show that, surprisingly, the converse of
\citeauthor{sriperumbudur13optimal}'s Theorem~\ref{thm:sri} is generally too
restrictive when $\sX$ is compact, but does hold when it is not. Also, they
confirm that the Polish assumption made in Theorem~\ref{thm:sri} is
superfluous.

\begin{remark}[On the significance of Theorem~\ref{thm:charac}]
\label{rem:significance_of_main}
One advantage of dropping the Polish assumption is that our result covers more
sets, especially non complete ones, such as open sets \emph{strictly} contained
in $\R^d$. Besides, we believe that dropping unnecessary hypotheses helps
clarifying the role of each remaining assumption. However, in our view, the
main contribution of Theorem~\ref{thm:charac} is its converse part, which
implies that many popular kernels \emph{fail} to metrize weak convergence. For
example, it rules out any RKHS contained in $\C{}{0}$ that maps some
probability measure(s) to 0.  This has important implications for the Stein
kernels adopted in
\citet{liu16stein,jitkrittum17linear,gorham17measuring,huggins2018random,feng2017learning,pu2017vae,liu16stein,chen18stein,chen2019stein,hodgkinson2020reproducing}
which, by design, map a particular target distribution to $0$ and which, if one
is not careful, will also induce RKHSes in $\C{}{0}$. 
\end{remark}

We now turn towards the proof. While it is almost obvious that metrization of
weak convergence implies separation of $\Mprob$, showing that it also implies
separation of $\Mf$ will require some work and, in light of
Lem.~\ref{lem:equi_topos}, is essentially all that remains to be proven. To do
so, we will use the following two lemmata. The first one is a straightforward
extension of a basic property of locally compact sets (every \emph{point} has a compact
neighborhood) from points to compact sets (\emph{every} compact set has a compact
neighborhood). The second shows that when $\Hk \subset \C{}{0}$ and $\sX$ is not
compact, then the RKHS metric cannot prevent some positive measures from
``diffusing'' to the null measure. This will imply that if $\k$ is not
characteristic to all finite measures, one can construct a sequence of
probability measures that converges in RKHS norm, but has some of its mass
diffusing to 0.

\begin{lemma}\label{lem:comp_neighborhood}
  Let $\sK$ be a compact subset of a locally compact space $\sX$. Then there
  exists an open neighborhood of $\sK$ with compact closure. Equivalently, 
  there exist an open set $\sO$ and a compact set $\sK'$ in $\sX$ such that
  $\sK \subset \sO \subset \sK'$.
\end{lemma}

\begin{lemma}\label{lem:diffusing_seq}
    Suppose that $\sX$ is not compact and that $\k$ is continuous with $\Hk
    \subset \C{}{0}$. Then there exists a sequence of probability measures
    $P_n$ such that $\normk{P_n} \to 0$. Moreover, for any compact $\sK \subset
    \sX$, one can additionally impose that $P_n(\sK) = 0$ for all $n$.
\end{lemma}


\begin{proof}[Proof of Lem.~\ref{lem:comp_neighborhood}]
Since $\sX$ is locally compact, every point has a compact neighborhood. So let us
consider the set of all compact neighborhoods of the points contained in $\sK'$.
Their interiors form an open cover of $\sK'$, and, since $\sK'$ is compact, a
finite number of them suffices to cover $\sK'$. Let $\sO$ be the finite union of
these interiors and $\sK'$ the union of their closures (i.e., the union of the
corresponding compact supersets). Then $\sO$ is open, $\sK'$ is compact, and $\sK \subset \sO \subset \sK'$ as advertised. 
We finally note
that this property is equivalent to the first claim (that there exists an open neighborhood of $\sK$ with compact closure) as $\sO$ is contained in a compact set if and only if its closure is compact.
\end{proof}

\begin{proof_no_qed}[Proof of Lem.~\ref{lem:diffusing_seq}]
    First we show that for any $\epsilon > 0$ and any integer $n > 0$, we can
    construct a sequence of $n$ points $\xx_1, \ldots, \xx_n$ in $\sX \back
    \sK$ such that for any $1 \leq i \neq j \leq n$, $| \k(\xx_i, \xx_j) | \leq
    \epsilon$. We will construct it one point at a time. Choose a point $\xx_1
    \in \sX \back \sK$. By assumption on $\k$, there exists a compact $\sK_1
    \subset \sX$ such that for any point $\xx \in \sX \back \sK_1$, $| \k(\xx,
    \xx_1) | \leq \epsilon$. Choose $\xx_2$ to be also outside of $\sK$, i.e.\
    $\xx_2 \in \sX \back (\sK \cup \sK_1)$ (non-empty, since $\sK \cup \sK_1$
    is compact and $\sX$ is not).  There exists a compact $\sK_2 \subset \sX$
    such that for any point $\xx \in \sX \back \sK_2$, $| \k(\xx, \xx_2) | \leq
    \epsilon$. Let $\xx_3$ be any point in $\sX \back (\sK \cup \sK_1 \cup
    \sK_2)$ (non empty because $\sX$ is not compact). Continue this procedure
    until point $\xx_n$. The sequence obviously satisfies the requirement.

    Now, for any integer $n > 0$, construct a finite sequence $\xx_1^{(n)},
    \ldots \xx_n^{(n)}$ such that for any $1 \leq i \neq j \leq n$, $|
    \k(\xx_i, \xx_j) | \leq \nicefrac{1}{n}$. Define the probability measures
    $P_n := \frac{1}{n} \sum_{i=1}^n \delta_{\xx_i^{(n)}}$. Then all $P_n(\sK)
    = 0$, since all $\xx_i^{(n)} \in \sX \back \sK$, and:
    \[
        \normk{P_n}^2 
            = \frac{1}{n^2} \sum_{1 \leq i \leq n} \k(\xx_i, \xx_i)
                +\frac{1}{n^2} \sum_{1 \leq i \neq j \leq n} \k(\xx_i, \xx_j)
            \leq \frac{n}{n^2} \norm{\k}_\infty
                + \frac{n (n-1)}{n^2} \frac{1}{n}
                \: \overset{n \rightarrow \infty}{\longrightarrow} \: 0.  \tag*{\myQED}
    \]
\end{proof_no_qed}

\begin{proof}[Proof of Thm.~\ref{thm:charac}]
%
Lemma~\ref{lem:equi_topos} yields the ``if'' part and the continuity of the
kernel in the converse. Assume now that $\k$ is not characteristic to $\Mf$.
Then there exists a non-zero, finite measure $\mu$ such that $f_\mu = 0$.
Let $\mu_+, \mu_-$ be its positive and negative parts respectively --~which are
mutually singular (Hahn decomposition). By renormalizing $\mu$ if needed, we
can assume without loss of generality that $\mu_-(\sX) \leq \mu_+(\sX) = 1$. If
$\mu_-(\sX) = \mu_+(\sX)$, then $\mu_-$ and $\mu_+$ are two non-equal
probability measures that are at RKHS distance 0, hence $\k$ does not metrize
weak convergence. So, for the sequel, assume that $\mu_-(\sX) < \mu_+(\sX)$.

Now, let $\sK$ be a compact subset of $\sX$ that satisfies $\mu_+(\sK) \geq
(\mu_-(\sX) + \mu_+(\sX)) / 2$, which exists because $\mu_+$ is regular and
$\mu_-(\sX) < \mu_+(\sX)$. 
Select now an open set $\sO$ and a compact set $\sK'$ satisfying $\sK \subset \sO \subset \sK'$, which exist by
Lemma~\ref{lem:comp_neighborhood}. Then, since $\sK \subset \sO$, $\mu_+(\sO)
\geq \mu_+(\sK)$.  Let now $P_n$ be probability measures as in
Lemma~\ref{lem:diffusing_seq} such that $P_n(\sK') = 0$ (and hence $P_n(\sO) =
0$) for all $n$. Consider the sequence of probability measures $\mu_n := \mu_- +
(1 - \mu_-(\sX)) P_n$. Then
\begin{align*}
    \normk{\mu_n - \mu_+}
        &= \normk{\mu_n - \mu_-} 
            \quad \text{(because  $f_{\mu_-} = f_{\mu_+}$)} \\
        &= (1 - \mu(\sX)) \normk{P_n} \quad \longrightarrow \quad 0,
\end{align*}
hence $\mu_n$ converges to $\mu_+$ in the RKHS metric. But $\mu_n$ does not
converge weakly to $\mu_+$ since 
\[
    \mu_+(\sO) \geq \mu_+(\sK) \geq (\mu_-(\sX) + \mu_+(\sX))/2 > \mu_-(\sX) \geq \mu_-(\sO) = \mu_n(\sO) \ ,
\]
which contradicts the Portmanteau lemma ($\lim\,\sup_n \mu_n(\sO) \not \geq
\mu_+(\sO)$).
\end{proof}

To prove that the initial claim (Claim~\ref{thm:simon}) is indeed wrong when
$\sX$ is not compact, it remains to show that being characteristic to $\Mf$ is
not equivalent to being characteristic to $\Mprob \subset \Mf$, i.e.\ that
there exists a kernel $\k$ with $\Hk \subset \C{}{0}$ that is characteristic to
$\Mprob$ but not to $\Mf$. We show this under the assumption that there already
exists a kernel of $\sX$ that is characteristic to $\Mf$, which is in
particular satisfied when $\sX$ is an open subset of $\R^d$.
\cj{And more generally, I think, when $\C{}{0}(\sX)$ is separable; see Guilbart.}  

\begin{proposition}\label{prop:exist}
    If there exists a bounded that is characteristic to $\Mf$ (with $\sX$
    compact or non compact), then there also exists a kernel $\k$ with $\Hk
    \subset \C{}{0}(\sX)$ that is characteristic to $\Mprob$ but not
    characteristic to $\Mf$. In particular, this $\k$ does not metrize the weak
    convergence of probability measures.
\end{proposition}

\begin{proof}
 Let $\kappa$ be any bounded kernel over $\sX$ that is \ispd, i.e.,
 characteristic to $\Mf$.
 \cj{TODO: if possible, prove that such a kernel must exist}
 $\xxi \in \sX$ and $g \in
\C{}{0}$ such that $g(\xxi) = 0$ and $g(\xx) > 0$ for any $\xx \neq \xxi$.
Consider $\k(\xx, \yy) := g(\xx) \kappa(\xx, \yy) g(\yy)$. Then $\k$ is a
kernel such that $\Hk \subset \C{}{0}$ (Lem.~\ref{lem:Hk_C0}) and $f_{\delta_\xxi}$ is the null
function, hence $\normk{\delta_\xxi} = 0$, so $\k$ is not \ispd But we will now
show that $\k$ is characteristic to $\Mf^0$, i.e.\ to $\Mprob$. Indeed, let
$\mu \in \Mf^0$ such that $\iint \k(\xx, \yy) \diff \mu(\xx) \diff \mu(\yy) =
0$. Since the product $g \mu$ is a finite measure and $\kappa$ is \ispd, the
previous equality implies that $g \mu$ is the null measure. Since $g > 0$ on
any $\xx \neq \xxi$, for any open set $\sO \subset \sX \back \{\xxi\}$,
$|\mu|(\sO) = 0$.  Hence the support of $\mu$ (well-defined, because $\mu$ is
regular) is contained in $\{\xxi\}$, i.e.\ $\mu$ is proportional to the Dirac
point mass in $\xxi$. Hence, if $\mu \in \Mf^0$, then $\mu$ is the null
measure.
\end{proof}

Proposition~\ref{prop:exist} has two implications.
First, it shows that the metrization condition in the non
compact case is \emph{strictly} stronger than in the compact case:
on compact spaces, some kernels do metrize weak convergence
\emph{without} separating all finite signed measures.
Second, combining it with Theorem~\ref{thm:charac}
shows that the alleged proof of Claim~\ref{thm:simon} must be flawed.
Another
confirmation will be given by point~\ref{cor:i} in
Corollary~\ref{cor:counterexamples}, with an explicit counter-example
constructed in its proof.
However, to strengthen our claim, we now explicitly point out the flaw in the proof
of Claim~\ref{thm:simon} by~\citet{simon18kde}.

\subsection{Flaw in the proof of Claim~\texorpdfstring{\ref{thm:simon} of \citeauthor{simon18kde}}{1 by Simon-Gabriel and Schölkopf}\label{sec:flaw}}

The flaw in the proof of Theorem~12 of \citet{simon18kde} (our Claim~\ref{thm:simon}) resides in their auxiliary Lemma~20, which is essentially
our Lemma~\ref{lem:equi_topos}, but without the assumption $\Hk \subset
\C{}{0}$. Their proof essentially consists in saying that, since $(P_\alpha)$
(denoted $(\mu_\alpha)$ there) is bounded, it is relatively vaguely compact, so
one can extract a subnet $(P_\beta)$ that converges vaguely to a measure $P'$
(denoted $\mu'$ there). They then try to identify the vague limit $P'$ with the
MMD- (or weak RKHS-) limit $P$ (denoted $\mu$ there) of the original net
$(P_\alpha)$, by arguing that weak and vague convergence coincide on $\Mprob$,
and that weak convergence implies MMD-convergence. Unfortunately, $\Mprob$ is
not closed in $\Mf$ for the vague topology, so nothing guarantees a priori that
$P' \in \Mprob$. And if $P' \not \in \Mprob$, then vague convergence to $P'$
does not imply weak convergence to $P'$ \citep[][Thm.~2.4.2]{berg84harmonic},
which is why the proof fails --~irremediably.

We can go further and exhibit a counter-example for the previous failure, i.e.\
a bounded, continuous, \ispd kernel and a sequence $(P_n)$ that converges to $P
\in \Mprob$ in MMD, but converges vaguely to another measure $P' \neq P$ in
$\Mf$. Indeed, consider the kernel $\kappa:= \k + 1$ from the proof of
Corollary~\ref{cor:counterexamples}\ref{cor:i} below. Let $\sK$ be a compact
neighborhood of $\xxi$ (which exists because $\sX$ is locally compact) and
choose a sequence $(P_n) \subset \Mprob$ as in Lemma~\ref{lem:diffusing_seq},
i.e.\ such that $\normk{P_n} \to 0$ and $P_n(\sK) = 0$ for all $n$. By using
the vague compactness of $\sB_+ := \{\mu \in \Mplus \,|\, \mu(\sX) \leq 1\}$
\citep[][Prop.2.4.6]{berg84harmonic} and extracting a subsequence if needed, we
may assume that $(P_n)$ converges vaguely to a measure $P' \in \sB_+$. Applying
Urysohn's lemma \citep[][Thm.~I-33]{villani10integration} to the compact set
$\{\xxi\}$ and an open neighborhood $\sO \subset \sK$ of $\xxi$, we get a
continuous function $f$ whose support is contained in $\sK$ and such that
$f(\xxi) = 1$. Since $f\in\C{}{0}$ and $P_n(f) = 0 < 1 = f(\xxi) =
\delta_{\xxi}(f)$, $P_n$ does not converge vaguely to $\delta_\xxi$, i.e.\ $P'
\neq \delta_{\xxi}$. Now $\kappa$ is bounded, continuous and \ispd, and induces
the same metric than $\k$ on $\Mprob$. So, since the KME of $\k$ maps the Dirac
measure $\delta_\xxi$ to the null function in $\Hk$ (see proof of
Prop.\ref{prop:exist}), we get 
\[
    \norm{P_n -\delta_{\xxi}}_{\kappa} 
        = \normk{P_n -\delta_{\xxi}} = \normk{P_n} \to 0 \ .
\]
Hence $(P_n) \to \delta_\xxi$ in MMD, but $(P_n)$ converges vaguely to a
different measure $P'$. 
\vspace{0.2cm}

\begin{remark}
The sequence $(P_n)$ converges neither weakly to $P'$ nor weakly to
$\delta_\xxi$, since weak convergence would imply vague and MMD convergence to
the same limit, i.e.\ would imply $P' = \delta_\xxi$. Hence $P'(\sX) \neq 1$
(otherwise, vague convergence would imply weak convergence, since both coincide
on $\Mprob$ \citep[][Cor.~2.4.3]{berg84harmonic}), and since $P' \in \sB_+$, we
get $P'(\sX) < 1$. So $(P_n)$ illustrates a phenomenon called \emph{mass
escaping at infinity}, which vague convergence, contrary to weak convergence,
cannot prevent.
\end{remark}

\section{General case: \texorpdfstring{$\sX$}{X} locally compact, non compact
and \texorpdfstring{$\Hk \not \subset \C{}{0}$}{Hk not in C0}}
\label{sec:general}

All previous sections assumed that $\Hk \subset \C{}{0}$ (automatically
satisfied when $\k$ continuous and $\sX$ is compact). So one may naturally
wonder whether this assumption could be dropped without replacement or at least
extended. Corollary~\ref{cor:counterexamples} shows that dropping it without
replacement is not possible; but Corollary~\ref{cor:extension} proposes a
slight extension.

\begin{corollary}\label{cor:counterexamples}
    The condition $\Hk \subset \C{}{0}$ 
    Theorem~\ref{thm:charac} cannot be replaced with $\Hk \subset \C{}{b}$ as, if $\sX$ is locally compact but
    not compact, then
    \begin{enumerate}[label=(\roman*)]
        \item \label{cor:i}there exists a bounded continuous kernel that is
        \ispd, but does not metrize the weak convergence of probability
        measures;
        \item \label{cor:ii}there exists a bounded, continuous, characteristic
        (to $\Mprob$) kernel that is not \ispd but metrizes the weak
        convergence of probability measures.
    \end{enumerate}
\end{corollary}
\begin{remark}
Note, however, that \emph{some} kernels with non-vanishing RKHS functions do satisfy the characterization of Theorem~\ref{thm:charac}. For example, Theorem~\ref{thm:charac} extends to any kernel of the form $k_c = k + c$ for $c > 0$ and $\Hk \not\subset \C{}{0}$, since $k_c$ and $k$ induce the same MMD.
\end{remark}

\begin{proof}
    \ref{cor:i} Let $\k$ be as in Proposition~\ref{prop:exist} and consider the
    new kernel $\kappa := \k + 1$. Then $\kappa$ is \ispd
    \citep[][Thm.~8]{simon18kde}, but $\kappa$ induces the same metric than $\k$
    on the set of probability measures $\Mprob$. Hence it does not metrize
    their weak convergence.

    \ref{cor:ii} Let $\xxi$ be a point in $\sX$, and $\k$ be as in
    Theorem~\ref{thm:charac}, i.e.\ a bounded, continuous kernel, with $\Hk
    \subset \C{}{0}$, that metrizes weak convergence over $\Mprob$. Then the
    new kernel $\kappa(\xx, \yy) := \ipdk{\delta_\xx - \delta_\xxi}{\delta_\yy
    - \delta_\xxi}$ is not \ispd (since the KME of $\delta_\xxi$ is the null
    function) but it induces the same RKHS metric than $\k$ on $\Mprob$,
    that is $\norm{P - Q}_{\kappa}=\normk{P- Q}$ for any $P,Q \in \Mprob$,
    hence metrizes weak convergence on $\Mprob$. (Remark: this implies that
    $\HH_\kappa \not \subset \C{}{0}$, which is also easy to check directly.)
\end{proof}

Let us mention that, in a side remark of \citet[][p.18]{guilbart78etude}, Guilbart already exhibits a theoretical
construction of kernels on $\R$ that are \ispd but do not metrize weak
convergence. Hence, Claim~\ref{thm:simon} was actually disproved before being
written.

We finish with a slight generalization of Theorem~\ref{thm:charac} that
encompasses some kernels whose RKHS is not contained in $\C{}{0}$. The result
builds on the same idea than in the proof of
Cor.~\ref{cor:counterexamples}\ref{cor:ii}.

\begin{corollary}\label{cor:extension}
    Suppose that $\sX$ is not compact and that $\Hk \subset \C{}{0}$. Fix $a
    \geq 0$ and $P \in \Mprob$ and define 
    \[
        \k_P^a(\xx, \yy) :=
            \ipdk{\delta_\xx - P}{\delta_\yy - P} + a
            = (\delta_\xx-P) \otimes (\delta_\yy - P) (\k) + a \ .
    \]
    Then $\k_P^a$ metrizes weak convergence of probability measures if and
    only if $\k$ is continuous and \ispd
\end{corollary}

\begin{proof}
    Since $\k_P^a(\xx, \yy) = \k(\xx, \yy) - f_P(\xx) - f_P(\yy) +
    \normk{P}^2 + a$, for any probability measures $S, T \in \Mprob$, we get
    \[
        \norm{S-T}_{\k_P^a}^2
            = (S-T)\otimes(S-T)(\k_P^a)
            = (S-T)\otimes(S-T)(\k)
            = \normk{S-T}^2 \ .
    \]
    Hence $\k$ and $\k_P^a$ define
    the same metric on $\Mprob$ and Thm.~\ref{thm:charac} concludes.
\end{proof}

%

\section{Conclusion}
\label{sec:conclusion}
MMDs are at the heart of machine learning solutions to a variety of fundamental tasks including two-sample testing, 
sample quality measurement and goodness-of-fit testing,  
learning generative models, 
de novo sampling and quadrature,  
importance sampling, 
and thinning. 
While these applications benefit from the tractability of MMDs compared to more classical probability metrics, the validity of their results depends critically on the MMD's ability to ensure weak convergence.
\citet{simon18kde} developed their Theorem~12 to provide a complete
characterization of weak-convergence metrization for MMDs with bounded
continuous kernels. However, our work shows that their characterization was
incorrect and provides an alternative result that fully characterizes the
weak-convergence metrization of MMDs with bounded $\C{}{0}$ kernels.
Surprisingly, we find that the compact and non compact cases are inherently
different, the latter requiring \emph{strictly} stronger conditions for the
metrization. This suggests that the question of weak-convergence metrization by
MMDs is more subtle than was previously thought. Our main results can also be
seen as a converse to \citeauthor{sriperumbudur13optimal}'s
Theorem~\ref{thm:sri}, which in particular show that many popular kernels,
particularly Stein kernels, can \emph{fail} to metrize weak convergence, if one
is not careful enough. In that spirit, we hope that our work will inform the
selection of appropriate kernels and MMDs in the future and launch new
inquiries into the metrization properties of other classes of MMDs.

\Needspace{4\baselineskip}
\acks{
CJSG was supported by the ETH Foundations of Data Science postdoctoral
fellowship and is associate fellow of the Center for Learning Systems (ETH/MPI
Túbingen).  AB was supported by the Department of Engineering at the University
of Cambridge, and
this material is based upon work supported by, or in part by, the U.S. Army Research Laboratory and the U. S. Army Research Office, and by the U.K. Ministry of Defence and the U.K. Engineering and Physical Sciences Research Council (EPSRC) under grant number [EP/R018413/2].
We declare no conflict of interests.}


\vskip 0.2in
\bibliography{main.bib}       





\end{document}